\documentclass[10pt,journal,compsoc]{IEEEtran}

\ifCLASSOPTIONcompsoc
    \usepackage[noadjust]{cite}
\else
  \usepackage{cite}
\fi

\usepackage{algorithm}
\usepackage{algorithmic}

\usepackage{booktabs}
\usepackage{amsmath}
\usepackage{amssymb}
\usepackage{mathtools}
\usepackage{amsthm}

\usepackage{multirow}
\usepackage{adjustbox}
\usepackage{float}
\usepackage{color}

\usepackage{physics}

\usepackage{url}            
\usepackage{amsfonts}       
\usepackage{nicefrac}       
\usepackage{microtype}      

\usepackage[capitalize,noabbrev]{cleveref}

\usepackage{wrapfig}
\RequirePackage[format=plain,labelformat=simple,labelsep=period,font=small,compatibility=false]{caption}
\RequirePackage[font=footnotesize,skip=3pt,subrefformat=parens]{subcaption}

\theoremstyle{plain}
\newtheorem{theorem}{Theorem}[section]

\newtheorem{lemma}[theorem]{Lemma}

\theoremstyle{definition}
\newtheorem{definition}[theorem]{Definition}

\theoremstyle{remark}

\makeatletter
\newtheorem*{rep@theorem}{\rep@title}
\newcommand{\newreptheorem}[2]{%
\newenvironment{rep#1}[1]{%
 \def\rep@title{#2 \ref{##1}}%
 \begin{rep@theorem}}%
 {\end{rep@theorem}}}
\makeatother

\newreptheorem{theorem}{Theorem}
\newreptheorem{lemma}{Lemma}
\newreptheorem{definition}{Definition}

\newcommand{\changes}[1]{{\color{black}{#1}}}
\DeclareMathOperator{\sign}{sign}

\begin{document}

\title{Pruning at Initialization -- A Sketching Perspective}

\author{Noga Bar,  Raja Giryes \IEEEmembership{Senior Member, IEEE}%
\IEEEcompsocitemizethanks{%
\IEEEcompsocthanksitem Noga Bar and Raja Giryes are affiliated with Tel Aviv University.\\
}%
}

\markboth{Pruning at Initialization-A Sketching Perspective}%
{}

\IEEEtitleabstractindextext{%
\begin{abstract}
The lottery ticket hypothesis (LTH) has increased attention to pruning neural networks at initialization.
We study this problem in the linear setting. We show that finding a sparse mask at initialization is equivalent to the sketching problem introduced for efficient matrix multiplication.
This gives us tools to analyze the LTH problem and gain insights into it.
Specifically, using the mask found at initialization, we bound the approximation error of the pruned linear model at the end of training.
We theoretically justify previous empirical evidence that the search for sparse networks may be data independent.
By using the sketching perspective, we suggest a generic improvement to existing algorithms for pruning at initialization, which we show to be beneficial in the data-independent case.  
\end{abstract}

}

\maketitle

\IEEEdisplaynontitleabstractindextext

\IEEEpeerreviewmaketitle

\section{Introduction}
\IEEEPARstart{P}{runing} a neural network at initialization, where weights are removed ahead of training, with minimal harm to network performance can be beneficial for training efficiency. It can also be used to gain insights into neural network training and expressivity as a whole.
According to the lottery ticket hypothesis (LTH) \cite{DBLP:conf/iclr/FrankleC19}, a network may contain extremely sparse subnetworks at initialization that achieve comparable or even better performance when trained in isolation.
The original algorithm suggested for finding the winning ticket was inefficient and required multiple trainings until convergence. Others, however, suggested pruning at initialization in a more efficient manner \cite{lee2018snip,tanaka2020pruning,wang2020picking,de2020progressive,alizadeh2022prospect,Lee2020A,sreenivasan2022rare}. 

Most works propose scoring functions for finding a subnetwork at initialization that rely on the specific data and task at hand \cite{lee2018snip,wang2020picking,de2020progressive,alizadeh2022prospect}.
Yet, evidence suggests that the winning lottery ticket is independent of data. 
Specifically, it has been shown that the winning ticket can be transferred between datasets and tasks \changes{\cite{evci2020rigging, Morcos19OneTicket,Chen_2021_CVPR,sabatelli2020transferability,you2022supertickets}}.
Moreover, previous work has shown that pruning methods can have good performance with corrupted training data \cite{su2020sanity}.
Furthermore, the work of SynFlow has demonstrated that a network can even be pruned at initialization without the use of data and a task-specific loss \cite{tanaka2020pruning}.

In this work, our aim is to explain the success of \changes{data-free} LTH . 
Previous theoretical analysis focused on a stronger version of the LTH \cite{ramanujan2020s}.
According to the study, deep neural networks (DNNs) have a sparse subnetwork capable of good performance even without training in a supervised setting.
Specifically, they show that for a given DNN architecture, initialization and target data (with labels), there is a subnetwork that achieves high accuracy without the need to train.
The theoretical explanation for the strong LTH is developed by estimating the function produced by the large and dense network using only the sparse subnetworks in it \cite{malach2020proving,pensia2020optimal,da2021proving}.
This result was further generalized and it has been shown that the initialization can be compatible with a set of functions over the training set \cite{burkholz2021existence}.
While the above results provide an intriguing explanation for the success of LTH in the supervised case, their explanation for the existence of the mask assumes the data is available for performing the pruning and that no training is required.  
Our study, on the other hand, focuses on a setting where the pruning is performed using random data and applied to other parameters than the ones at initialization (possibly after training).

For our analysis, we draw a connection between pruning at initialization and a well known sketching algorithm \cite{drineas2006fast}.
Originally, the algorithm was suggested for efficient matrix multiplication while minimizing the error of the multiplication approximation.
We choose to focus on the linear case and analyse it in order to gain insights into the general case, which is a common practice when analyzing neural networks \cite{arora2018optimization,brutzkus2018sgd,shamir19Exponential,arora2019implicit,li2020towards,arora2018a,Nacson22Implicit,yun2021a}. 

Our key insight is that in the linear case, sketching corresponds to pruning at initialization. 
Using this relation, we extend the sketching analysis to the approximation error at initialization with an unknown vector. 
In our case, this vector can be viewed as the learned vector at the end of training.

We focus on the link between sketching and pruning without data.
This allows analyzing the performance of data-free pruning methods. Specifically, 
we develop a bound that shows that the pruning error without data depends on the ratio between the weights of the linear network at initialization and at the end of training.
The differences between the parameters are small in practical NNs and especially in the Neural Tanget Kernel (NTK) regime \cite{jacot2018neural}.

Equipped with the theoretical results, we turn to study practical algorithms for the problem of pruning at initialization without data. 
We consider the state-of-the-art and unsupervised SynFlow method and show that in the linear case it bears great similarity to sketching when it is applied with a random vector. This provides us with a possible explanation for SynFlow's success. 
We also analyse the connection of the supervised SNIP method to sketching and use it to suggest a version of SNIP without data. 

The relations we draw also suggest that successful pruning is highly correlated with selecting weights that have large magnitude at initialization. 
We empirically validate our findings for both SynFlow and the iterative magnitude pruning (IMP) algorithm (the method suggested in the original LTH work) showing that they both have the tendency to maintain large magnitude weights after pruning.

To further validate our analysis with random data, we show that various pruning algorithms that use the input data \cite{DBLP:conf/iclr/FrankleC19,wang2020picking,lee2018snip} only mildly degrade when the input is replaced with completely random data. It is consistent with previous evidence that pruning methods do not fully rely on data \cite{su2020sanity}.

Based on the above, we suggest a general improvement to pruning algorithms in the \changes{when the data is unavailable}.
Instead of pruning weights by removing the lowest scores, sketching masks are randomized based on some probability. We test the effect of replacing the strict threshold criterion of pruning with a randomized one in which the mask is sampled based on the pruning method scoring. This strategy shows improvement in most cases for various network architectures and datasets, namely, CIFAR-10, CIFAR-100 \cite{krizhevsky2009learning}, Tiny-ImageNet \cite{wu2017tiny} \changes{and ImageNet \cite{deng2009imagenet}}.

Our main contributions are (i) connecting pruning and sketching; (ii) providing error bound for \changes{data-free pruning} at initialization in the linear case;  (iii) relating SynFlow and SNIP to sketching \changes{and enhancing them in the absence of data}; (iv) presenting a general improvement \changes{for pruning without data}.

\section{Related Work}

The main DNN pruning approach follows the train$\rightarrow$prune$\rightarrow$fine-tune pipeline \changes{\cite{mozer1988skeletonization,molchanov2016pruning,lecun1989optimal,hassibi1993optimal,han2015learning,guo2016dynamic,bartoldson2020generalization}}. It requires training until convergence, then pruning and finally fine tuning. This saves computations at inference time.
Another approach sparsifies the model during training where training and pruning are performed simultaneously \changes{\cite{chauvin1988back,carreira2018learning,louizoslearning,bellecdeep,mocanu2018scalable,mostafa2019parameter,back2023magnitude,you2022supertickets}}. In this case, the fine tuning stage is omitted.
The main gain of these methods is also at inference time.
Lastly, pruning at initialization, which is our focus, aims to zero parameters at initialization. 
Such methods improve efficiency in parameters for both training and inference \cite{evci2020rigging}. Additionally, it can be used for neural architecture search \changes{\cite{mellor2021neural,white2021powerful,abdelfattah2021zero,you2022supertickets}} and for gaining a deeper theoretical understanding of DNN \cite{arora2018stronger}.

In this work, we discuss pruning by simply zeroing weights (unstructured pruning). Yet, another approach is to remove complete neurons (structured pruning) 
\changes{\cite{novikov2015tensorizing,jaderbergspeeding,chenotov2,chen2021only,neill2020overview,rachwan2022winning,li2022revisiting,bartoldson2020generalization,ganjdanesh2022interpretations}}.
These methods generally include more parameters than unstructured ones but are usually more beneficial for computational time on standard hardware.
Unstructured pruning can lead to reduced computations for some hardware while maintaining a low number of parameters \cite{evci2020rigging}.
Note that the concepts presented in our work may be extended to structured pruning as
some method originally presented for unstructured pruning that employ scoring has been extended for structured pruning at initialization \cite{van2020single,rachwan2022winning}.

Due to the exponential search space, pruning DNNs at initialization is a challenging task.
Before LTH, it was suggested to use SNIP which prunes the NN while maintaining its connectivity according to the magnitude of the gradient \cite{lee2018snip}.
It was improved by using it iteratively \cite{de2020progressive} or applying it after a few training steps \cite{alizadeh2022prospect}. It was shown that one may reduce the number of required iteration by using a simpler data, e.g., small fraction of the data \cite{paul2022lottery}.
It was also proposed to improve the signal propagation in the DNN at initialization \cite{Lee2020A} and to find a mask while preserving gradient flow using the Hessian matrix \cite{wang2020picking}. 
Zhang et al. \cite{zhang2021validating} proposed a sparse pruning algorithm for high-dimensional manifolds and a theoretical validation of subnetworks' viability. 
Those methods rely on the gradient of their parameters w.r.t. the input data.
Yet, a study found that some of these hardly exploit information from the training data \cite{su2020sanity} hinting that it may not be required for finding good sparse subnetworks.
Liu et al. \cite{liu2020finding} relaxed the need for supervision and employed a teacher-student model with unlabeled data.

LTH \cite{DBLP:conf/iclr/FrankleC19} increased the interest in sparse neural networks.
The original study suggested an exhaustive and supervised algorithm for finding the winning ticket.
Other works reduce data dependency when searching for the winning ticket.  when searching for the winning ticket. 
It was demonstrated that looking for tickets at early stages of training, prior to convergence, leads to improved performance and saves computational overhead \cite{you2019drawing,hubens2021one,shen2022prune}.
It was further shown that training the LT with partial datasets leads to decent winning tickets \cite{zhang2021efficient}.
Early pruning has been theoretically proven to be beneficial \cite{wolfe2021provably}.
Together with other pruning methods at initialization, it was shown that the winning tickets do not rely heavily on the training data \cite{su2020sanity}. Even so, Frankle et al. \cite{frankle2020linear} have shown that unsupervised pruning mainly finds the support of the weights and may be inferior to the supervised case.
Our study complements these works by drawing a relationship to sketching that provides bounds for data-free pruning.
It also provides an adaptation of supervised methods to reduce the requirement of input data and improve their pruning strategy.

Other works demonstrated the generality of the winning ticket and showed that a single pruned network can be transferred across datasets and achieve good performance after fine-tuning \cite{evci2020rigging, Morcos19OneTicket,Chen_2021_CVPR,you2022supertickets} and even that LT can be used for non-natural datasets \cite{sabatelli2020transferability}. Universal winning tickets that fit multiple functions were shown theoretically in \cite{burkholz2021existence}.
LTH was strengthened and it was suggested that a sparse subnetwork within the neural network initialization has high accuracy without training \cite{ramanujan2020s}.

The strong LTH was later theoretically studied \changes{\cite{malach2020proving,pensia2020optimal,orseau2020logarithmic,fischer2021towards,fischer2022plant,da2021proving}}.
Note that strong LTH requires a larger network before pruning. This was relaxed by initializing the weights iteratively \cite{chijiwa2021pruning}.
These works only bound the approximation error at initialization and assume labeled data.

SynFlow \cite{tanaka2020pruning} is not data or task dependent.
We examine its equivalence to a sketching method when the input data to SynFlow is random rather than a ones vector.
The works in \cite{patil2021phew,gebhart2021unified,pham2024towards} examined active paths in networks at initialization and established other unsupervised pruning methods.
Recently, it was suggested to prune by approximating the spectrum of a neural tangent kernel in a data agnostic way \cite{wang2022ntk}.

We establish equivalence between the sketching problem and pruning.
To this end, we use a well known Monte-Carlo technique for efficient approximation of matrix multiplication and compression designed for handling large matrices \cite{drineas2006fast}.
We analyze network pruning methods through the `sketching lens', which leads to an approximation error bound of the mask at initialization. The type of result is similar to the bounds of approximation for the strong LTH but using a different proof technique that analyze the ``weak'' case without data.

\section{Sketching and Pruning at Initialization}
\label{sec:SketchPrune}

{\bf Notations.}
The input data we aim to model is $x_i\in \mathbb{R}^d$ for $i=1,..., n$ with the corresponding labels $y_i\in\mathbb{R}$.
We relate to the input as a matrix $X\in\mathbb{R}^{d\times n}$ where $x_i$ are its columns and $y\in\mathbb{R}^n$ is the vector containing the labels.
The $i$th row in a matrix $A$ is denoted as $A_{(i)}$ and the $i$th column as $A^{(i)}$.
Unless otherwise stated, $\norm{x}$ is the Euclidean norm of $x$. 

{\bf Problem statement.} In this work, we relate to linear features where we aim to find a vector that approximates the data, i.e. $w$ such that $X^T w \simeq y$.
Our main interest is to sparsify $w$ while minimizing the mean squared error of the features. The optimization problem can be written as
\begin{align}\label{eq:prob_stat}
    \min_{m, s.t. \norm{m}_0 \leq s} \norm{X^Tw - X^T(w \odot m)},
\end{align}
where $\norm{\cdot}_0$ is the number of non-zeros in $m$ and $\odot$ stands for entry-wise multiplication. Clearly, when the 'weighted mask' $w$ is applied to a vector, it holds that $\norm{w \odot m}_0 \leq s$.

{\bf Connecting sketching and pruning at initialization.}
We focus on a sketching approach presented initially for efficient matrix multiplication \cite{drineas2006fast}. 
Its goal is multiplying only a small subset of columns and rows, while harming the approximation accuracy as little as possible.
The main contribution of the approach is in the way this subset is chosen.
We adapt it to match the linear features settings, where the matrix multiplication is reduced to be a multiplication of given input $X\in\mathbb{R}^{d\times n}$ and a vector $w\in\mathbb{R}^d$. This results in the linear features $X^Tw$.

\begin{figure}
\vspace{-0.15in}
    \begin{algorithm}[H]
      \caption{Sketching for mask \cite{drineas2006fast}.}
      \begin{algorithmic}\label{algo:sketch_mask}
        \STATE \textbf{Input:} Probability $p\in\mathbb{R}^d$ and density $s\in\mathbb{Z}^+$.
        \STATE \textbf{Initialization} Set mask  $m=0$.
        \FOR{$t=1,...,s$}
          \STATE $i_t \sim p$ 
          \STATE $m_{i_t} = m_{i_t} + \frac{1}{s p_{i_t}}$
        \ENDFOR
        \STATE \textbf{Output:} The mask $m$.
      \end{algorithmic}
    \end{algorithm}
  \vspace{-0.4in}
\end{figure}

Note that choosing rows/columns in matrix multiplication is equivalent to choosing a mask for entries in $w$ for matrix-vector multiplication.
Given a choice of the entries, the entries that are compatible with the mask are non-zero while others are zero.
The zero entries in $m$ leads to ignoring whole columns in the matrix. For example, the approximation of the multiplication of a matrix $A$ with the vector $b$ masked with $m$ satisfies $Ab \simeq A(b\odot m) = A^{(\{i, m_i\ne 0\})} b_{(\{i, m_i\ne 0\})}$.
This means that any sketching algorithm that compresses matrix multiplication by selecting rows/columns can also be used as a masking procedure for a vector in matrix-vector multiplication.

In order to find a mask $m$ for the vector $w$ as described in the problem statement (\cref{eq:prob_stat}), we use the sketching algorithm presented in \cref{algo:sketch_mask}.
The algorithm samples an entry randomly according to the distribution $p$ and sets this entry to be non-zero until the desired density is achieved.
Note that the indices $i_t$ are sampled i.i.d. with return and that the distribution of $p$ is used both for sampling and scaling.
Due to the possibility of sampling the same entry more than once, the resulting $m$ has granularity of $\norm{m}_0 \leq s$ and it is not necessarily holds that $\norm{m}_0=s$. Additionally, not all non-zero entries in the mask are equal. 

By \cite[Lemma 4]{drineas2006fast}, the optimal probability for minimizing the error is
\begin{align}\label{eq:prob}
    \Pr_{w,X}[i] = \frac{\norm{X_{(i)}} \abs{w_i}}{\sum_{j=1}^d \norm{X_{(j)}} \abs{w_j}}, \quad i=1,...,d.
\end{align}
The probability is proportional to the norms of the rows of $X$, and it corresponds to a specific entry at each example $x_k$ across all inputs, $k=1,...,n$. Recall that $X_{(i)}\in\mathbb{R}^n$.
Thus, masking an entry $i$ in $w$ means ignoring all the $i$th entries in the data $x_k$, $k=1,...,n$.
Note that when sampling with \eqref{eq:prob}, if the inputs are uniformly distributed, then it is more likely to choose weights with larger magnitudes given the data. \cref{fig:norm} shows empirically that DNNs' pruning methods also tend to keep higher magnitude weights.

\begin{figure}[t]
       \vspace{-0.15in}
    \centering
    \begin{subfigure}{0.48\linewidth}
    \centering
    \includegraphics[width=\textwidth]{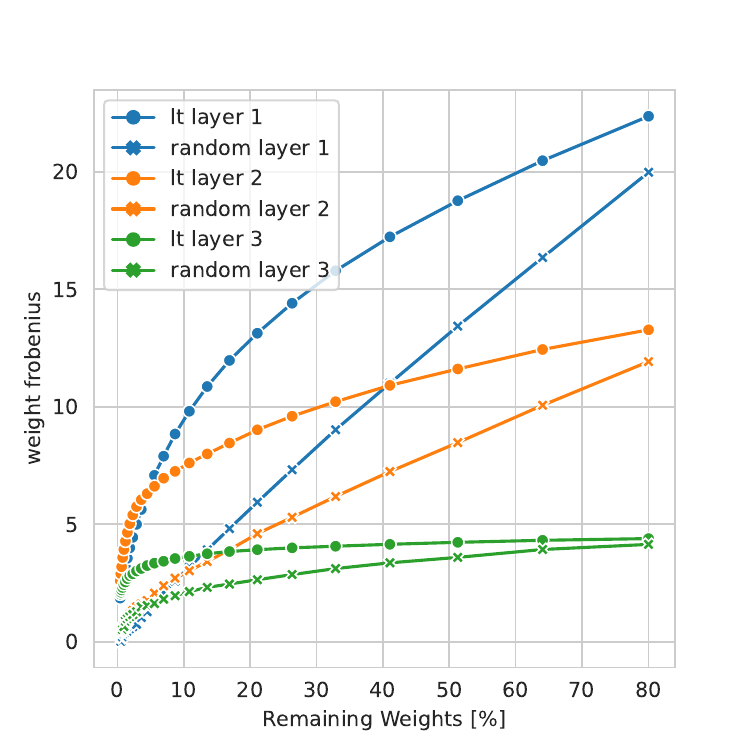}
    \end{subfigure}
    \begin{subfigure}{0.48\linewidth}
    \centering
    \includegraphics[width=\textwidth]{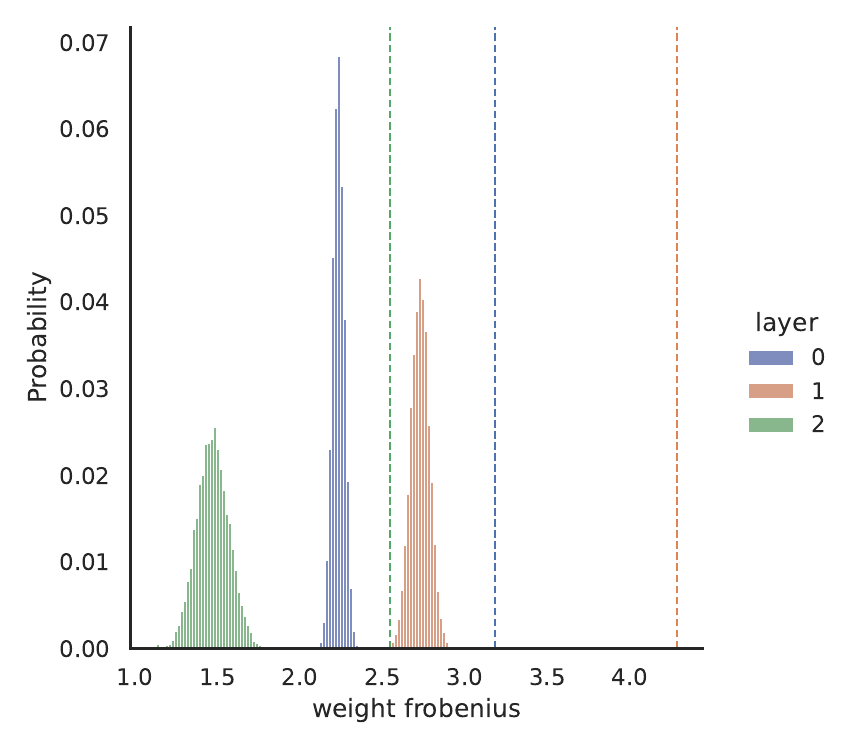}
    \end{subfigure}
        \vspace{-0.12in}
    \caption{Pruning a fully connected NN with Fashion-MNIST: (left) Norms of winning lottery tickets and random masks with multiple sparsities. (right) Winning tickets norms vs. 10,000 random masks with 1.2\% density.
    } \label{fig:norm}
    \vspace{-0.2in}
\end{figure}

\section{Pruning Approximation Error}
\label{sec:approx_error}
Using the relationship between pruning and sketching drawn above, 
we use the analysis tools from sketching for examining the properties of pruning.
All proofs are in the sup. mat.

In pruning at initialization, the initial vector $w_0$ is used for finding the mask, and then we approximate the features with the same mask and another unknown vector $w^\star$. Typically, $w^\star$ is the learned weights at the end of training.
The features are $X^T w^\star$ and we aim to minimize $\norm{X^T w^\star - X^T (w^\star \odot m)}$.
Note that since $m$ is chosen randomly we bound the expected value of the approximation error.
For simplicity, we denote $p^0_{i} = \Pr_{w^0,X}[i]$, having $p^0$ as the optimal probability to sample the mask at initialization.
In our analysis, $X$ may be given or randomly distributed while the weighting $w$ is assumed to be known.

First, we present a simplified version of the one found in the original sketching paper \cite{drineas2006fast} assuming the data $X$ are given. We rephrase the lemma to match the vector-matrix multiplication case.

\begin{lemma}{\cite[Lemma 4]{drineas2006fast}}\label[lemma]{lemma:sketching}
Suppose $X\in\mathbb{R}^{d\times n}$, $w^0\in\mathbb{R}^d$ and  $s\in \mathbb{Z}^+$ then using \cref{algo:sketch_mask} with $p^0$ for $m$ then the error is
\begin{align*}
    & \mathbb{E}_{m|X}\left[\norm{X^T w^0 - X^T (w^0 \odot m)}^2\right] 
    \\
    & ~~~~~~~~~~~~~~~~~~~~~~~~~~= \frac{1}{s} \left(\sum_{k=1}^d \norm{X_{(k)}} \abs{w_k^0}\right)^2 - \frac{1}{s}\norm{X^T w^0}^2.
\end{align*}
\end{lemma}

The lemma bounds the error of sparse approximation at initialization given the data. It is proportional to $\frac{1}{s}$, which means that as the mask density grows the error decreases, as expected.
This connection will be consistent throughout our analysis.
Note that the above result holds for post-training pruning when the mask is found according to 
 $w^\star$ which are the weights after training instead of the initial weighting.
We extend the result to the case of finding the mask according to the initialization $w^0$ and apply it on $w^\star$.
Next, we establish a bound of the error when $X$ is drawn from a normal i.i.d. distribution, where we also observe this behavior.
\begin{lemma} \label[lemma]{lemma:sketching_rand_data}
Suppose $X\in\mathbb{R}^{d\times n} \sim \frac{1}{\sqrt{n}} \mathcal{N}(0, I)$, $w^0\in\mathbb{R}^d$ and  $s\in \mathbb{Z}^+$ then when using \cref{algo:sketch_mask} with $p^0$ for drawing $m$,  the error can be bounded as follows
\begin{align*}
    \mathbb{E}_{X}\left[\norm{X^T w^0 - X^T (w^0 \odot m)}^2\right] \leq \frac{1}{s} \norm{w^0}^2.
\end{align*}

\end{lemma}

\begin{figure}[t]
    \vspace{-0.15in}
    \centering
    \includegraphics[width=0.6\linewidth]{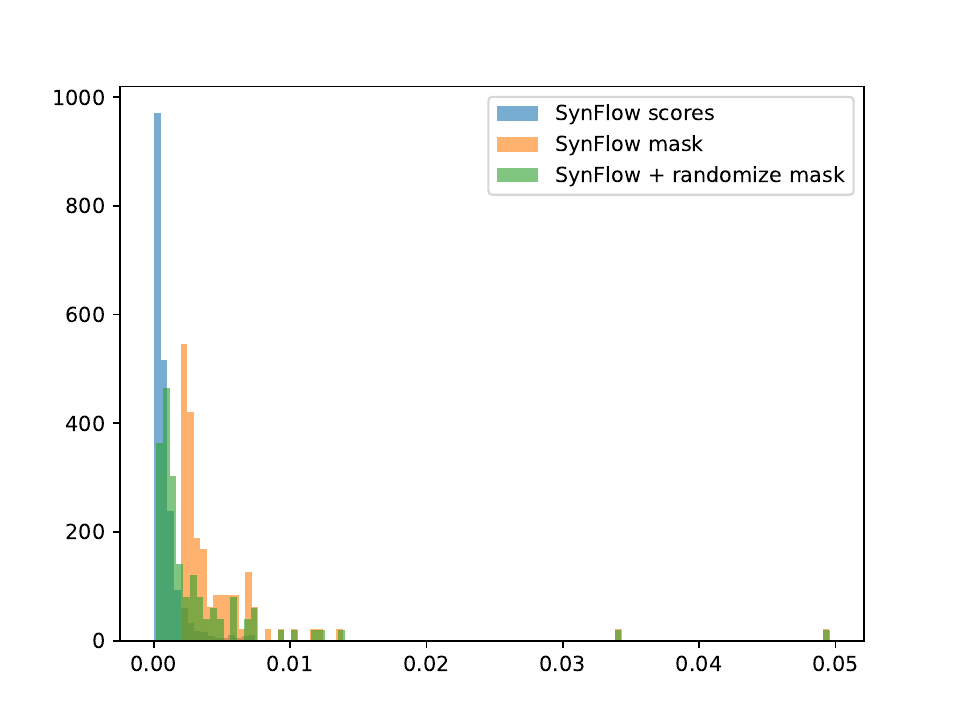}
    \vspace{-0.17in}
    \caption{Histogram of scores of NN before pruning and the weights chosen by SynFlow with/without randomization.}
        \label{fig:scores_hist}
\end{figure}

\begin{figure}[t]
    \vspace{-0.15in}
    \centering
    \begin{subfigure} {0.48\linewidth}
    \centering
    \includegraphics[width=\textwidth]{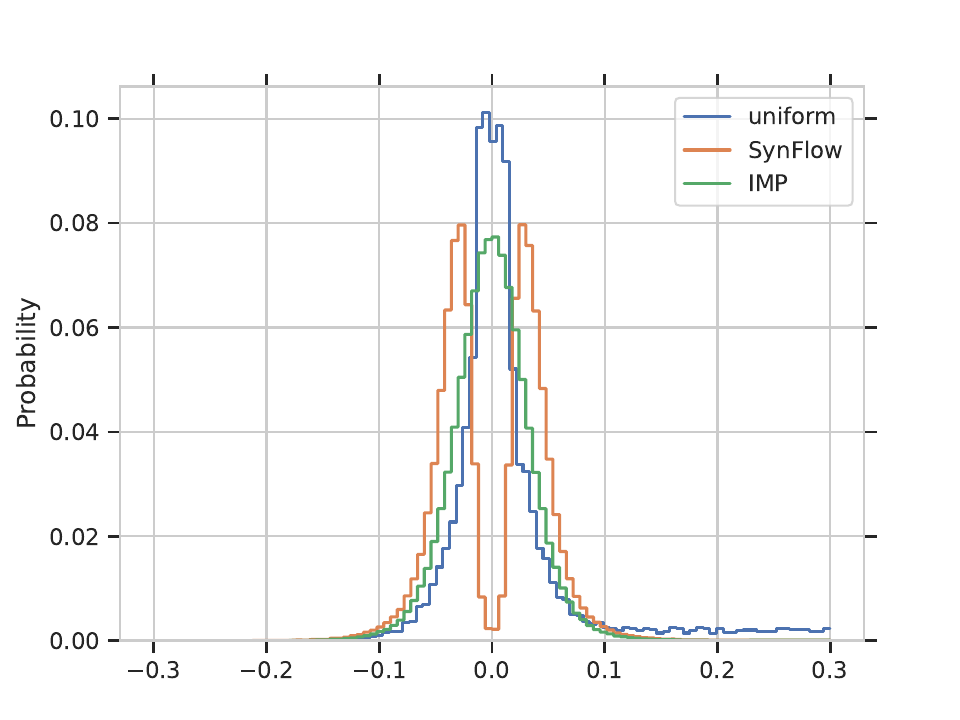}
    \end{subfigure}
    \begin{subfigure}{0.48\linewidth}
    \centering
    \includegraphics[width=\textwidth]{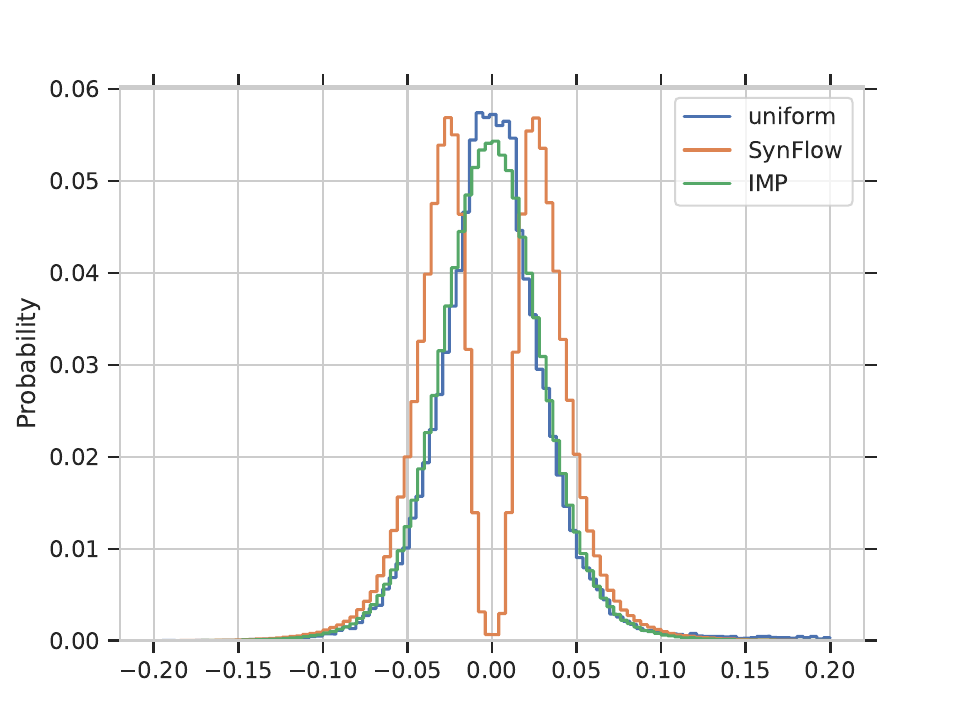}
    \end{subfigure} 
    \vspace{-0.15in}
    \caption{Weights histogram in sparse subnetworks at initialization for VGG-19 and CIFAR-10 for 2\% (left) and 5\% remaining weights. SynFlow and IMP have bias to large magnitude weights compared to a uniformly random mask.}
    \label{fig:init_mag}
\end{figure}

Note that the above error bounds hold at initialization. Yet, in the pruning at initialization setting, unlike matrix multiplication, we do not aim to apply a mask on a known parameter $w^0$ but on some unknown parameter $w^\star$ established after a learning procedure.
Thus, we bound the approximation error with the mask when we apply it to some unknown vector $w^\star$ and normally distributed data.

Before we establish our main result, we provide a lemma regarding the error, when the mask is found with some initial vector $w^0$ and data $\Tilde{X}$ but applied on other input data $X$ and weight vector $w^\star$.
\begin{lemma}\label{lemma:error_x}
    Suppose $X,\Tilde{X}\in\mathbb{R}^{d\times n}$, $w^0,w^\star\in\mathbb{R}^d$ and  $s\in \mathbb{Z}^+$ then when using \cref{algo:sketch_mask} with $p^0$ and $\Tilde{X}$ for $m$, the error can be bounded as follows
\begin{align*}
    & \mathbb{E}_{m}\left[\norm{X^T w^\star - X^T (w^\star \odot m)}^2\right] \\
    & ~~~~~~~~~~~~~~~~~~~~~~~~~~ \leq \frac{1}{s} \sum_{k=1}^d \frac{\sum_{j=1}^d \norm{\Tilde{X}_{(j)}}|w^0_j|}{\norm{\Tilde{X}_{(k)}}|w^0_k|} \norm{X_{(k)}}^2 {w^\star_k}^2.
\end{align*}
\end{lemma}

Next, we show the main result about the error of the mask for unknown parameters and random data. 

\begin{theorem} \label{thm:2weight_rand_data}
Suppose $X\in\mathbb{R}^{d\times n} \sim \frac{1}{\sqrt{n}} \mathcal{N}(0, I)$, $w^0,w^\star\in\mathbb{R}^d$ and  $s\in \mathbb{Z}^+$. Let $c = \max_j \abs{\frac{w^\star_j}{w^0_j}}$. When using \cref{algo:sketch_mask} with $p^0$ for $m$  the error can be bounded as follows
\begin{align*}
    &\mathbb{E}_{X}\left[\norm{X^T w^\star - X^T (w^\star \odot m)}^2\right]  \leq \frac{c^2}{s} \norm{w^0}_1^2 
\end{align*}

\end{theorem}

Note that in post-training pruning, mask selection optimizes the left-hand-side (lhs) of Thm. \ref{thm:2weight_rand_data}, where $w^\star$ (the weights at the end of training) is given. 
The theorem bounds the lhs error when the mask $m$ is selected using the weights $w_0$, i.e., the weights at initialization.  
To ensure that the bound is meaningful, we compare to a simple baseline:
The error bound when the mask is sampled uniformly at random.

\begin{lemma} \label[lemma]{lemma:unimask_rand_data}
Let $X\in\mathbb{R}^{d\times n} \sim \frac{1}{\sqrt{n}} \mathcal{N}(0, I)$, $w^0\in\mathbb{R}^d$, $s\in \mathbb{Z}^+$ and $c = \max_j \abs{\frac{w^\star_j}{w^0_j}}$. Then when choosing $m$ uniformly at random (i.e., using \cref{algo:sketch_mask} with a uniform distribution) the error obeys
\begin{align*}
    & \mathbb{E}_{X}\left[\norm{X^T w^\star - X^T (w^\star \odot m)}^2\right] 
    \leq \frac{d}{s} \norm{w^\star}^2 \leq \frac{dc^2}{s} \norm{w^0}^2
\end{align*}
\end{lemma}

According to \cref{lemma:unimask_rand_data}, we have an additional dimension factor $d$ that is avoided in Thm.~\ref{thm:2weight_rand_data} and for uniform choice the bound is $\ell_2$ instead of $\ell_1$.
The ratio between the norms satisfies the inequalities $\norm{w_0} \leq \norm{w_0}_1 \leq \sqrt{d} \norm{w_0}$. The last equality holds only when $w_0$ is uniform. In this case a uniform choice of a mask is ideal (since all the entries in $w_0$ are equivalent). In addition, uniform initialization is not a common practice. 

\begin{table*}[t]
\centering
\begin{adjustbox}{width=\linewidth}
\begin{tabular}{cc|cccc|cccc}
\toprule
Model                            &  Density         & SynFlow                   & SynFlow + $\chi$ & \begin{tabular}[c]{@{}c@{}}SynFlow +\\rand. mask \end{tabular} & \begin{tabular}[c]{@{}c@{}}SynFlow +\\rand. mask + $\chi$\end{tabular}          & \begin{tabular}[c]{@{}c@{}}SNIP +\\normal data\end{tabular} & \begin{tabular}[c]{@{}c@{}c@{}}SNIP +\\rand. mask + \\normal. data\end{tabular} & \begin{tabular}[c]{@{}c@{}}SNIP +\\sparse data\end{tabular} & \begin{tabular}[c]{@{}c@{}c@{}}SNIP + \\sparse data + \\rand. mask\end{tabular}                      \\ \midrule
\multirow{3}{*}{VGG-19}& 10\% & 93.01          $\pm$ 0.19 & \textbf{93.12 ± 0.12} & 93.06 $\pm$ 0.12 &  \textbf{93.12 ± 0.18} & 92.15 ± 0.07 & 92.55 ± 0.57 & 92.85 ± 0.06 &  \textbf{93.16 ± 0.36}     \\
                      & 5\%  & \textbf{92.68 $\pm$ 0.11} & 92.52 ± 0.20          & 92.44 $\pm$ 0.11 &  92.63 ± 0.2           & 91.53 ± 0.04 & 92.18 ± 0.78 & 92.02 ± 0.18 &  \textbf{92.79 ± 0.78}     \\
                      & 2\%  & 91.68 $\pm$ 0.28          & 91.32 ± 0.11          & 91.63 $\pm$ 0.16 &\textbf{91.71 ± 0.28}   & 90.35 ± 0.31 & 91.04 ± 0.14 & 90.97 ± 0.11 & \textbf{ 92.03 ± 0.91}     \\ \midrule
\multirow{3}{*}{ResNet-20}& 10\% & 86.55 $\pm$ 0.18 & 86.70 ± 0.32 & \textbf{86.74 $\pm$ 0.22} &  86.59 ± 0.29          & 85.69 ± 0.37 & 86.77 ± 0.41 & 85.85 ± 0.07 & \textbf{87.02 ± 0.65}     \\
                         & 5\%  & 83.19 $\pm$ 0.36 & 83.28 ± 0.31 & \textbf{83.55 $\pm$ 0.38} &  83.45 ± 0.42          & 82.37 ± 0.69 & 83.59 ± 0.4 & 82.29 ± 0.01  & \textbf{83.72 ± 0.73}     \\
                         & 2\%  & 77.06 $\pm$ 0.35 & 77.10 ± 0.12 & \textbf{77.74 $\pm$ 0.51} &  \textbf{77.74 ± 0.43} & 77.2 ± 0.56  & 78.99 ± 0.73 & 76.98 ± 0.66 & \textbf{79.39 ± 0.49}      \\ \midrule \midrule
\multirow{3}{*}{VGG-19}           & 10\% & 69.90 $\pm$0.26  & 69.84 ± 0.15 & 69.97          $\pm$0.43 & \textbf{70.24±0.29} & \textbf{72.75±0.22} & 72.67 ± 0.01 &  72.13  ±  0.08  &  72.49  ± 0.33           \\
                                 & 5\%  & 68.25 $\pm$0.66  & 68.39 ± 0.31 & \textbf{68.65 $\pm$0.33} & 68.32±0.08           &71.48±0.72 & 71.34 ± 0.03  &  71.05  ±  0.11  &  \textbf{71.51 ±  0.45}          \\
                                 & 2\%  & 65.98 $\pm$0.38  & 65.68 ± 0.23 & 65.87          $\pm$0.50 & \textbf{66.00±0.22}           &1.00 	   & 1.00 &  48.48  ±  13.88   &  \textbf{53.05 ±15.38}          \\ \midrule
\multirow{3}{*}{ResNet-20}        & 10\% & 50.66 $\pm$0.78  & 50.97 ± 0.51 & \textbf{52.29 $\pm$0.64} & 51.80±0.60  &66.90±0.33 & 67.11 ± 0.62 & 67.32±0.13 & \textbf{67.52 ±  0.10}             \\
                                 & 5\%  & 41.12 $\pm$0.92  & 40.87 ± 0.76 & \textbf{42.92 $\pm$0.53} & 42.44±0.48  &60.97±0.45 & 61.03 ± 0.03 & 61.94±0.60 & \textbf{62.68 ±  0.32}             \\
                                 & 2\%  & 23.39 $\pm$1.04  & 23.52 ± 0.27 & \textbf{24.89 $\pm$0.75} & 24.35±0.72  &47.61±4.15 & 48.86 ± 0.07 & 50.30±0.04 & \textbf{51.76 ±  0.36}             \\ 
\bottomrule 
\end{tabular}
\end{adjustbox}
\vspace{-0.1in}
\caption{Accuracy results for vanilla SynFlow \cite{tanaka2020pruning}, SynFlow with $\chi$ distributed input and with mask randomization; and accuracy results for SNIP \cite{lee2018snip} with normal random data, sparse input data and mask randomization. Above are results with CIFAR-10 and below are results with CIFAR-100. More results appear in \cref{table:random_data}.} \label{table:synflow_cifar}
\vspace{-0.2in}
\end{table*}

\section{Using Sketching With Existing Pruning Methods}
We demonstrate the results shown in earlier sections (\cref{sec:approx_error,sec:SketchPrune}) within the context of existing techniques for pruning during initialization.
Our emphasis is on exploring SynFlow \cite{tanaka2020pruning}, an unsupervised method, alongside SNIP \cite{lee2018snip}. For SNIP, we propose an adaptation of the method to operate in an unsupervised manner.

\subsection{SynFlow and Sketching}
We turn to analyse theoretically the connection between SynFlow and sketching in the linear case.
We calculate the scores given by SynFlow and treat them as probabilities.

SynFlow performs a forward pass on the model with a vector of ones $\textbf{1}$ as input and the absolute values of the initialized parameters.
The scores are calculated by multiplication with $\textbf{1}$, which sums up all the output features:
\begin{align}\label{eq:synflow}
    R_{SF} = \textbf{1}^T f(\textbf{1}; \abs{w}).
\end{align}
We turn to show that by changing the input, SynFlow resembles sketching.
Instead of $\textbf{1}$ as input, we use $\norm{X_{(i)}}$. Thus,
\begin{align}\label{eq:synflow_chi}
    R_{SF} = \textbf{1}^T f(\norm{X_{(i)}}; \abs{w}) .
\end{align}
Note that in the linear model case, where $w\in \mathbb{R}^d$, we have
\begin{align*}
    R_{SF} = \left(\norm{X_{(1)}}, \norm{X_{(2)}}, ..., \norm{X_{(i)}}, ..., \norm{X_{(d)}} \right)\abs{w}.
\end{align*}
Thus, in this case, SynFlow score for each weight in $w$ is:
$$ \frac{\partial R_{SF}}{\partial \abs{w}_i} \odot \abs{w}_i = \norm{X_{(i)}} \abs{w}_i ,$$
which yields that when looking at the scores as a distribution $p_i\propto \norm{X_{(i)}} \abs{w}_i$. Note that in sketching, the probability $p_i^0$, as defined in \cref{eq:prob}, is proportional to the same term and yields the lowest expected approximation error \cite{drineas2006fast}.
Thus, it means that in the linear features setting the SynFlow scores and the optimal probability for sketching share the same properties.
Hence, if one uses randomization over the mask with SynFlow scores as in \cref{algo:sketch_mask}, we get the sketching algorithm. 
Note that in \cref{thm:2weight_rand_data}, we assume random normal data as input.
Under these assumptions, according to \cref{eq:synflow_chi}, SynFlow inputs are $\sqrt{\sum_{i=1}^n \frac{1}{n} x_i^2}$, where $x_i$ are random normal, which implies a $\chi$ distribution vector for its input.
To summarize, our analysis suggests that if we want to make SynFlow more similar to sketching, we should apply it with random sampling and $\chi$ distribution for its input. We demonstrate the advantage of this in the experiments.

\begin{table*}[t]
\centering
\begin{adjustbox}{width=0.9\textwidth}
    \begin{tabular}{l|ccc|ccc}
\toprule
Model                       & \multicolumn{3}{c|}{VGG-19}                       & \multicolumn{3}{c}{ResNet-20}                     \\ 
Density                     & 10\%           & 5\%            & 2\%            & 10\%           & 5\%            & 2\%            \\\midrule
Random & 91.14 ± 0.08 & 89.24 ± 0.22 & 86.28 ± 0.04 & 85.04 ± 0.12 & 72.08 ± 1.38 & 10.0            \\
Magnitude only (w/o data) & 93.05 ± 0.22 & 92.23 ± 0.9 & 91.28 ± 0.43 & 84.88 ± 0.7 & 81.97 ± 0.52 & 75.15 ± 0.86           \\\midrule
GraSP (supervised) & 92.21 ± 0.69 & 91.68 ± 1.21 & 90.92 ± 1.28 & 87.28 ± 0.04 & 84.37 ± 0.21 & 79.53 ± 0.06 \\
GraSP + normal data & \textbf{92.23 ± 0.95} & 91.23 ± 0.91 & 90.9 ± 1.63 & 86.3 ± 0.46 & 83.61 ± 0.13 & \textbf{79.59 ± 0.1} \\
GraSP + normal data + rand. mask & 92.12 ± 0.88 & \textbf{91.57 ± 1.1} & \textbf{91.22 ± 1.08} & \textbf{86.59 ± 0.29} & \textbf{84.35 ± 1.34} & 78.36 ± 0.16        \\\midrule
IMP (supervised) & 93.56 ± 0.19 & 92.64 ± 0.19 & 89.63 ± 1.57 & 89.63 ± 0.24 & 85.99 ± 1.3 & 82.46 ± 0.73 \\ 
IMP+rand. data & 91.53 ± 0.32 & 91.42 ± 0.19 & \textbf{90.58 ± 0.69} & 85.46 ± 0.11 & \textbf{83.07 ± 0.11} & \textbf{77.97 ± 0.04}  \\
IMP+rand. data+rand. mask & \textbf{91.57 ± 0.03} & \textbf{91.45 ± 0.12} & 90.55 ± 0.7 & \textbf{85.5 ± 0.15} & 82.77 ± 0.13 & 77.54 ± 0.49 \\ \midrule
\changes{ProsPr (supervised) } & 93.41 ± 0.05 & 93.0 ± 0.14 & 92.15 ± 0.17 & 87.08 ± 0.59 & 82.12 ± 0.66 & 73.04 ± 1.79\\
\changes{ProsPr  + rand. data }& \textbf{93.37 ± 0.08} & 93.06 ± 0.06 & \textbf{92.79 ± 0.07} & \textbf{87.61± 0.19} & 83.38 ± 0.45 & 74.48 ±0.75\\
\changes{ProsPr + rand. data + rand. mask} & 93.34 ± 0.22 & \textbf{93.15 ± 0.15} & 92.47 ± 0.04 & 87.5±0.19 & \textbf{83.69±0.42} & \textbf{75.55± 0.55}\\
\midrule \midrule
Random & 66.02 ± 0.13 & 62.7 ± 0.57 & 58.28 ± 0.25 & 51.09 ± 0.43 & 26.61 ± 3.51 & 1.0 ± 0.0 \\
Magnitude & 69.43 ± 0.49 & 68.09 ± 0.21 & 65.18 ±0.47  & 52.03 ± 0.24 & 45.64 ± 0.4 & 32.9 ± 0.65  \\
 \midrule
GraSP (supervised)                      & 70.64 ± 0.43 & 70.03 ± 0.55 & 67.78 ± 0.42 & 67.24 ± 0.23 & 63.08 ± 0.33 & 53.18 ± 0.94           \\
GraSP+normal data            &  70.52 ± 0.18 & 70.06 ± 0.29 & 68.07 ± 0.16 & 67.37 ± 0.02 & 63.66 ± 0.06 & \textbf{54.15 ± 0.63} \\
GraSP+normal data+rand. mask & \textbf{71.15 ± 0.32}   & \textbf{70.08 ± 0.01}          & \textbf{68.53 ± 0.18}           & \textbf{67.39 ± 0.46}           & \textbf{64.34 ± 0.15} & 53.79 ± 0.16 \\ \midrule
IMP (supervised) & 70.88 ± 0.96 & 69.87 ± 1.28 & 67.9 ± 1.36 & 58.38 ± 2.97 & 50.39 ± 0.16 & 40.88 ± 0.22       \\
IMP + rand. data & 66.84 ± 0.16 & \textbf{66.23 ± 0.17} & 54.81 ± 6.87 & 51.66 ± 0.68 & 43.84 ± 0.17 & 32.92 ± 0.65 \\
IMP + rand. data + rand. mask & \textbf{67.1 ± 0.21} & 65.3 ± 0.23 & \textbf{59.05 ± 0.14} & \textbf{52.05 ± 0.6} & \textbf{44.68 ± 0.38} & \textbf{33.37 ± 0.74}           \\ \midrule
\changes{ProsPr (supervised)} &  59.12 ± 2.86 & 52.61 ± 2.72 & 52.63 ± 1.36  & 44.57 ± 2.18 & 30.6 ± 3.49 & \textbf{26.30 ± 0.50} \\
\changes{ProsPr  + rand. data} & 11.55 ± 2.11 & 8.43 ± 0.58 & 8.39 ± 0.15 & 45.41 ± 3.5 & 37.15 ± 2.32 & 22.69 ± 0.22\\
\changes{ProsPr + rand. data + rand. mask} &\textbf{72.68 ± 0.16} & \textbf{71.38 ± 0.27} & \textbf{69 ± 0.36} & \textbf{53.91 ± 1.16} & \textbf{42.69 ± 0.33} & \textbf{27.45 ± 0.60}\\
\bottomrule
\end{tabular}
\end{adjustbox}
\vspace{-0.1in}
\caption{Average accuracy results with GraSP \cite{wang2020picking}, Iterative Magnitude Pruning (IMP) \cite{DBLP:conf/iclr/FrankleC19} and ProsPr \cite{alizadeh2022prospect} with random data used for pruning and mask randomization that replaces the thresholding. Reported accuracy results with CIFAR-10 (top) and CIFAR-100 (bottom).
Note that the supervised versions are provided only as a reference.
} \label{table:random_data}
\vspace{-0.1in}
\end{table*}

\subsection{SNIP and Sketching} \label{sec:snip}
In this section we turn to analyze another well-known method of pruning at initialization named SNIP \cite{lee2018snip}, which aims to estimate the importance of each weight according to its magnitude and gradient magnitude at initialization.
The magnitude of the gradient is calculated with respect to input data, $\mathcal{D}$.
SNIP assigns the following saliency score of mask at index $j$:
$g_j(w;\mathcal{D}) = \abs*{\frac{\partial L(m\odot w;\mathcal{D})}{\partial m_j}},$
where $L$ is the loss used for training and the value of the mask before pruning is $1$ in all its entries.
We analyze the case of $\ell_1$ loss for inputs $X$. 
Previously, a statistical justification for using $\ell_1$ loss for DNN classification was proven \cite{janocha2017loss} and it was further
shown that using $\ell_1$ loss can be beneficial for robust classification \cite{ghosh2017robust}.
The weights' sailency score in the linear model is
\begin{align*}
& g_j(w; X,y) = \frac{1}{n} \abs*{\frac{\partial \sum_{i=1}^n \abs{x_i^T (w\odot m - y_i)}}{\partial m_j}} \\
    & ~~~~~~~~~~~~~~~~~~
= \frac{1}{n} \abs{w_j} \abs*{ \sum_{i=1}^n \sign(x_i^Tw - y_i) x_{ij}} .
\end{align*}
Again, we look for input data that causes SNIP to behave like sketching, which is a well-established field. 
Interestingly, if the data have only one non-zero entry in each row of $X$, i.e. $\norm{X_{(j)}}_0 = 1$, it holds that the induced probability is $p_j = \frac{\abs{w_j}{\abs{x_{ij, \text{ s.t. } x_{ij} \ne 0}}}}{\sum_{k=1}^d \abs{w_k} \abs{x_{ik, \text{ s.t. } x_{ik} \ne 0}}} = \frac{\abs{w_j} \norm{X_{(j)}}}{\sum_{k=1}^d \abs{w_k} \norm{X_{(k)}}}$.
Note that this probability is optimal from a sketching perspective. Thus, we suggest to apply SNIP with random sparse data and random sampling of the mask. 
We validate this empirically in the experiments.

\section{Neural Tangent Kernel Pruning}
\label{sec:prune_NTK}
We turn to apply our sketching results to the NTK regime \cite{lee2019wide,jacot2018neural,arora2019fine},
which was introduced for examining the training dynamics of DNNs.
Under the infinite width assumption, the gradient descent steps over a DNN become analytically tractable and similar to learning with a known kernel.

NTK assumptions allow representing DNN features in a linearized manner.
Thus, we can apply our sketching results above. To this end, we provide some basic NTK definitions.

{\bf NTK definitions.} We use the same notations as appeared in \cite{lee2019wide}: $\theta_t \in \mathbb{R}^d$ is the vectorization of the parameters at time $t$; $f_t(x)\in\mathbb{R}^m$ is the output of the model at time $t$;  $(\mathcal{X}, \mathcal{Y})$ are the input vectors and labels (we assume that $\mathcal{Y}_i\in\mathbb{R}$); $f_t(\mathcal{X})$ is a concatenation of all outputs;
$n$ is the notion of width of the model; $J(\theta_t) = \nabla_\theta f_t(\mathcal{X}) \in \mathbb{R}^{\abs{\mathcal{X}}m\times d}$ is the Jacobian;  $\hat{\Theta}_t = \nabla_\theta f_t(\mathcal{X})\nabla_\theta f_t(\mathcal{X})^T = \frac{1}{n} J(\theta_t)J(\theta_t)^T$ is the empirical NTK. We refer to the analytic kernel as $\Theta = \lim_{n\rightarrow\infty} \hat{\Theta}_0$.

For the claims in \cite{lee2019wide} to hold we use assumptions that are based on the original paper: \textbf{(i)} The empirical NTK converges in probability to the analytic NTK: $\hat{\Theta}_0\rightarrow_{n\rightarrow\infty}\Theta$; \textbf{(ii)} The analytic NTK $\Theta$ is full rank. $0<\lambda_{\min} \leq \lambda_{\max} < \infty$ and let   $\eta_{critical} = 2(\lambda_{\min} + \lambda_{\max})^{-1}$; \textbf{(iii)} $(\mathcal{X}, \mathcal{Y})$ is a compact set and for $x,\Tilde{x}\in\mathcal{X}$, $x\ne\Tilde{x}$; 
and \textbf{(iv)} The Jacobian is locally Lipschitz as defined in \cref{def:local_lip} (originally stated in \cite[Lemma 2]{lee2019wide}):
\begin{definition}[Local Lipschitzness of the Jacobian]\label[definition]{def:local_lip}
Denote $B(\theta_0, C) = \{\theta : \norm{\theta - \theta_0}\leq C\}$.
The Jacobian is locally Lipschitz if there exists $K>0$ such that for every $C>0$, with high probability over random initialization the following holds for $\theta,\Tilde{\theta}\in B(\theta_0, C)$
\begin{align*}
    \norm{J(\theta) - J(\Tilde{\theta})}_F \leq K\norm{\theta-\Tilde{\theta}}  \quad \text{and} \quad 
    \norm{J(\theta)}_F \leq K.
\end{align*}
\end{definition}
Note that local Lipschitzness constants are usually determined by the DNN activation functions.

{\bf NTK Pruning.}
We focus on the linearized approximation $f_t^{lin}(x) = \nabla_{\theta_t} f_t(x) \theta_t$ of the model features.
We aim to find a mask to apply on $\theta_t$ according to the linear features of the network at initialization, $f_0^{lin}(\mathcal{X})$ and $\theta_0$, using sketching as in \cref{algo:sketch_mask}.
Thus, we can use our previous analysis to bound the error induced by the mask on the features at time $t$, $f_t^{lin}(x)$. The following theorem relies on the NTK bounds in \cite[Theorem G.4]{lee2019wide}.
Proof is in the sup. mat.

\begin{theorem}\label{thm:NTK}
    Under assumptions [i-iv], $\delta_0 > 0$ and $\eta_0\leq \eta_{critical}$,
    let $F(A) = \sum_{i=1}^d \frac{1}{\norm{A^{(i)}}}$ and $m\in\mathbb{R}^d$ be a $s$-dense mask found with \cref{algo:sketch_mask} with $p$ according to $f_0^{lin}(\mathcal{X})$ and $\theta_0$ (\cref{eq:prob}).
    There exist $R_0 >0$, $K>1$ such that for every $n\geq N$ the following holds with probability $>1-\delta_0$ over random initialization when applying GD with learning rate of $\eta_0$:
    \begin{align*}
        &\mathbb{E}_m\left[\norm{f_t^{lin}(\mathcal{X}) - \nabla_{\theta_t} f_t(\mathcal{X}) (\theta_t \odot m)}^2\right] \leq \\
        & \frac{1}{s} K^3 \norm{\theta_0}_1 F(J(\theta_0)) \left(\norm{\theta_0}_1 + F(\theta_0) \frac{9K^4R_0^2}{\lambda_{\min}^2} + 6\sqrt{d} \frac{K^3R_0}{\lambda_{\min}}\right).
    \end{align*}
\end{theorem}
Note that for a mask density of $s = O(\sqrt{d})$, the error of using the mask at time $t$ only depends on the initialization and constants from the NTK assumptions.
Additionally, $\nabla_{\theta_t} f_t(\mathcal{X}) (\theta_t \odot m)$ are the masked linearized outputs of the model at time $t$.
We prove the above theorem using the bound on the Jacobian norm and the claim that under the conditions of \cite[Theorem G.4]{lee2019wide} it is guaranteed that the distance of the parameters at time $t$, $\theta_t$, from the parameters at initialization, $\theta_0$, is bounded.

\section{Experiments}
\label{sec:experiments}
We study our theoretical insights empirically on sparse DNNs.
We tested multiple pruning at initialization methods: SynFlow \cite{tanaka2020pruning}, SNIP \cite{lee2018snip}, GraSP \cite{wang2020picking}, Iterative Magnitude Purning (IMP), the algorithm suggested for finding the winning ticket in \cite{DBLP:conf/iclr/FrankleC19}, \changes{and ProsPr \cite{alizadeh2022prospect}}. 

\begin{table}
\centering
\begin{adjustbox}{width=0.8\linewidth}
\begin{tabular}{cc|cc}
\toprule
Model                       & Density & SynFlow & \begin{tabular}[c]{@{}c@{}}SynFlow +\\random mask\end{tabular} \\ \midrule
\multirow{3}{*}{ResNet-18}   & 10\%    & 58.3    & \textbf{60.64}                                                     \\
                            & 5\%     & 57.63   & \textbf{58.57}                                                     \\
                            & 2\%     & 54.62   & \textbf{55.56}                                                     \\ \midrule
\multirow{3}{*}{WideResNet} & 10\%    & 59.25   & \textbf{60.22}                                                     \\
                            & 5\%     & 57.49   & \textbf{58.84}                                                     \\
                            & 2\%     & 54.8    & \textbf{55.54}   \\ \midrule\midrule
\changes{\multirow{2}{*}{ResNet-50}}  & \changes{10\%   } & \changes{62.76}   & \changes{\textbf{63.05}}                \\
                         & \changes{5\%}     & \changes{57.65}   &  \changes{\textbf{57.99}} \\ \bottomrule  
\end{tabular}
\end{adjustbox}
\vspace{-0.1in}
\caption{SynFlow + random mask results for Tiny-ImageNet (top) \changes{ and ImageNet (bottom).}} \label{table:tiny_imagenet}
\end{table}

\begin{table}
\centering
\begin{adjustbox}{width=0.8\linewidth}
\changes{
\begin{tabular}{cc|cc}
\toprule
Model                       & Density & SynFlow & \begin{tabular}[c]{@{}c@{}}SynFlow +\\random mask \end{tabular} \\ \midrule
\multirow{3}{*}{Kaiming}   & 10\%  & 93.01 $\pm$ 0.19  &  \textbf{93.06 $\pm$ 0.12}                                                         \\
                            & 5\%     & \textbf{92.68 $\pm$ 0.11}   &  92.44 $\pm$ 0.11                         \\
                            & 2\%     & \textbf{91.68 $\pm$ 0.28 }  & 91.63 $\pm$ 0.16              \\ \midrule
\multirow{3}{*}{Normal} & 10\%    & 92.16 $\pm$0.18   & \textbf{92.68 $\pm$ 1.03  }                                                   \\
                            & 5\%     &  91.11 $\pm$ 0.01 &  \textbf{92.22 $\pm$ 1.02}                                                    \\
                            & 2\%     &  90.27 $\pm$ 0.07  &  \textbf{91.44 $\pm$ 1.09}  \\ \midrule
\multirow{3}{*}{Xavier} & 10\%    & 91.5 $\pm$ 0.17   &  \textbf{93.33 $\pm$ 0.18}                                                   \\
                            & 5\%     &  91.57 $\pm$ 0.02  & \textbf{93.11 $\pm$0.14} \\ 
                            & 2\% & 91.31 $\pm$ 0.05 & \textbf{92.03 $\pm$ 0.42} \\
                            \bottomrule  
\end{tabular}}
\end{adjustbox}
\vspace{-0.1in}
\caption{The effect of mask randomization under various initialization methods. 
Note that pruning methods during initialization are executed based on the initial weights, and our mask randomization technique demonstrates robustness against variations in initialization. Across all initialization schemes, we either match or enhance performance, whereas SynFlow leads to performance degradation.} \label{table:init_method}
\vspace{-0.18in}
\end{table}

We use CIFAR-10/100 \cite{krizhevsky2009learning}, Tiny-ImageNet \cite{wu2017tiny} and \changes{ImageNet \cite{deng2009imagenet}} with VGG-19 \cite{simonyan2014very}, ResNet-20 \cite{he2016deep}, WideResNet-20-32 \cite{zagoruyko2016wide} \changes{and ResNet-50}. 
For CIFAR-10/100 and Tiny-ImageNet, we employ SGD with momentum, batch size $128$, and learning rate $0.1$ multiplied by $0.1$ after epoch $80$ and $120$.
We train the models for $160$ epochs with weight decay $10^{-4}$.
\changes{For ImageNet we use $90$ epochs, batch size of $100$ and the learning rate is multiplied by $0.1$ after epochs $30$, $60$ and $80$.}
\changes{Unless otherwise is stated we use Kaiming initialization \cite{he2015delving}. }
For Tiny-ImageNet, we use a modified version of ResNet-18 and WideResNet-18.
Our code is based on the repositories of \cite{open-lth,tanaka2020pruning,su2020sanity,alizadeh2022prospect}.
All the experiments performed on a single NVIDIA GeForce RTX 2080 Ti.
Our code is attached in the sup. mat.

We found the SynFlow mask using $100$ iterations. We employ SNIP, GraSP \changes{and ProsPr} with batch size of $256$.
For IMP, we use $1000$ iterations of warmup for VGG-19 and $2000$ iterations for ResNet-20 with CIFAR-10 and $6000$ for CIFAR-100 with ResNet20.
In the \changes{data-free pruning} case, examples are drawn from a normal distribution with the same expectation and standard deviation as the original dataset.
Results include the average and standard deviation of 3 different seeds.

Given that in sketching the mask is selected randomly, we suggest doing the same for the purning at initialization techniques. 
Our proposed mask randomization strategy is performed in three steps. The first step is computing the threshold according to the number of desired remaining weights.
We compute it globally according to all the scores in the network.
The second is to compute the number of remaining weights within each layer with respect to hard thresholding. Finally, the third step is to randomly select a mask for each layer according to the granularity found in the previous step and the scores of the pruning method.
We randomize the mask in this manner due to computational limitations.
In Appendix~\ref{sec:layerwise_sparse} we present results where the sparsity is imposed in a layerwise manner and all the layers have the same sparsity.

\cref{fig:init_mag} shows that the weights' magnitudes chosen by SynFlow and IMP are larger than a uniform random choice of parameters.
Note that SynFlow has a stronger bias towards larger magnitudes than IMP although the weight magnitude is part of IMP's scoring function. 
Fig. \ref{fig:norm} presents a case, where the IMP mask has an extremely high norms compared to random masks and the high norm is maintained across sparsities.
We tested the effect of randomizing the mask on the chosen scores. Fig. \ref{fig:scores_hist} compares the histogram of scores of weights chosen with SynFlow with and without randomization compared to the distribution of all the scores in the network.
We report the scores of VGG-19 with CIFAR-10 and for simplicity we present the results only on a subset of $1000$ weights chosen uniformly at random.

To show the effect of changes used in the theoretical analysis we tested SynFlow with mask randomization and replacing the all-one vector with $\chi$ distributed random data.
We generate the $\chi$ distribution according to the $\ell_2$ norms of a vector with normally distributed variables and dimension $n=128$.
We randomize data at each pruning iteration. 
To validate empirically the relation drawn between SNIP and sketching above, we tested SNIP with sparse random data. We choose randomly for each input pixel location, a single image in the batch where the pixel is non-zero.
Note that the images are drawn from a normal distribution and are not natural images.
\cref{table:synflow_cifar} shows that replacing the thresholding with a random mask and replacing the input with the distribution derived from our analysis improve results in most cases when compared to the original SynFlow and naive data-independent SNIP.
For SynFlow, the use of mask randomization leads to improved performance.
The combination of sparse data and mask randomization leads to superior accuracy for SNIP.
It suggests that our sampling of mask and data can be used and improve other pruning methods based on it, e.g. \cite{alizadeh2022prospect,de2020progressive}.
Additional results with the original supervised SNIP are in Table~\ref{tbl:sup_snip} (in Appendices).
Most accuracy results with random masks and sparse data outperform SNIP even with labeled input.

Our claim that finding a good sparse subnetwork at initialization does not necessarily depend on the data is tested by an experiment where the masks are learned with completely random input and then retrained with real supervised data.
\cref{table:random_data} reports CIFAR-10/100 accuracy results. We provide the supervised methods only as a reference and do not boldface them.
We can see that randomized masks produce comparable and usually better results than pruning methods with random data.
Also, it can be seen that, indeed, the performance of data dependent pruning methods is only partially explained by the use of supervised input data, as shown in \cite{su2020sanity}.
Note that accuracy degradation is expected since those methods are designed to work with labeled data.

For Tiny-ImageNet, the improvement in performance is clear with mask randomization inspired by sketching; see \cref{table:tiny_imagenet}.
The improvement in accuracy is around $1\%$ for all sparsities and architectures.
\changes{Mask randomization is also beneficial for performance with ImageNet.}
Overall, randomizing the masks leads to a new the state-of-the-art with SynFlow.

Our analysis is done when the weights are given and there is a dependence on the initial value of the weights. Therefore, we tested that mask randomization based on scores is robust to different initialization distributions.
\Cref{table:init_method} shows the results for Kaiming \cite{he2015delving}, Normal and Xavier \cite{glorot2010understanding} initializations with CIFAR-10 and VGG-19.
Note that the benefit of mask randomization is not limited to a specific initialization method.

\section{Conclusion and Future Work}
This work presents and analyzes the task of pruning at initialization from a sketching perspective.
Based on our analysis, we gain a new justification for the claim that good sparse subnetworks are data independent.
Additionally, we apply ideas from sketching to DNNs.
Our proposed changes can be added to existing and future pruning methods.

{\bf Limitation.} While we have proposed a novel connection between pruning at initialization and sketching, our framework has some limitations. We apply it only in the setting of pruning at initialization. 
One might investigate our approach to other settings of pruning after or during training and structured pruning methods. For example, to draw a relationship to the latter, one may extend our analysis of SNIP to be with structured sparsity.
Moreover, we have studied only one sketching approach. 
We leave for future work exploiting more complex sketching methods, which possibly can lead to new insights and improvements to DNN pruning.
Also, future research may be conducted to generalize the rather simple linear features framework to a deeper model where the search space for the mask is even larger.
The search space in this case might be ambiguous, i.e., two different masks can lead to the same output \cite{zheng20identifiability}.
\changes{We leave for future research to theoretically analyse the effect of initialization on pruning in the lens of sketching.}
We believe that the relation we draw here can be used in many other directions besides those indicated above and contribute to DNN understanding and pruning.

\bibliographystyle{IEEEtran}
\bibliography{ref}

\appendices

\begin{table*}[t]
\centering
\caption{Accuracy results with the original supervised SNIP and SNIP with sparse data and mask randomization. Above are results with CIFAR-10 and below are results with CIFAR-100.}
\label{tbl:sup_snip}
\begin{adjustbox}{width=0.75\textwidth}
        \begin{tabular}{c|ccc|ccc}
    \toprule
    Model                       & \multicolumn{3}{c|}{VGG-19}                       & \multicolumn{3}{c}{ResNet20}                     \\ 
Density                     & 10\%           & 5\%            & 2\%            & 10\%           & 5\%            & 2\%            \\\midrule
    SNIP (supervised)           & 92.62 ± 0.09 & 91.66 ± 0.25 & 90.92 ± 0.16 & 86.29 ± 0.92 & 82.55 ± 1.02 & 78.16 ± 0.8  \\
    SNIP + sparse data + rand. mask                     & \textbf{93.16 ± 0.36} & \textbf{92.79 ± 0.78} & \textbf{92.03 ± 0.91} & \textbf{87.02 ± 0.65} & \textbf{83.72 ± 0.73} & \textbf{79.39 ± 0.49} \\
    \midrule
    SNIP (supervised)           & 71.80 ±  0.66 & 71.07 ± 0.23 & \textbf{56.05 ± 13.97} & 67.03 ± 0.04 & 61.80 ± 0.38 & 49.42 ± 0.65 \\
    SNIP + sparse data + rand. mask & \textbf{72.49  ± 0.33} & \textbf{71.51 ±  0.45} & 53.05 ±15.38 & \textbf{67.52 ±  0.10} & \textbf{62.68 ±  0.32} & \textbf{51.76 ±  0.36}\\
    \bottomrule
    \end{tabular}
\end{adjustbox}
\end{table*}

\section{Results with Layerwise Sparsity}
\label{sec:layerwise_sparse}
We include results with SNIP \cite{lee2018snip} when the wanted sparsity is imposed in a layerwise manner, i.e. for a given density $s$\%, $s$\% of the weights remain in each layer.
Usually, pruning is performed globally and different layers remain with different sparsities. In that case, the requirement is just that $s$\% of all the weights will remain in the network. In the layerwise case, the requirement is imposed on each layer separately.
The results of pruning under this setup are detailed in \cref{tbl:snip_layerwise} and demonstrate the improvement of using mask randomization also in this case.
\begin{table*}[t]
\centering
\caption{Accuracy results with SNIP method without data when sparsity is imposed in a layerwise manner.}
\label{tbl:snip_layerwise}
\begin{adjustbox}{width=0.4\textwidth}
\changes{
        \begin{tabular}{c|ccc}
    \toprule
    Model                       & \multicolumn{3}{c}{VGG-19}                    \\ 
Density                     & 10\%           & 5\%            & 2\%           \\\midrule
    SNIP           & 85.12 ± 10.27 & 78.26 ± 4.85 & 64.64 ± 10.0  \\
    SNIP + rand. mask                     & \textbf{91.12 ± 0.32} & \textbf{89.64 ± 0.15} & \textbf{87.97 ± 0.12} \\
    \bottomrule
    \end{tabular}
    }
\end{adjustbox}
\end{table*}

\section{Useful Lemmas and Proofs}
First we present useful lemmas and proofs that are used to prove the lemmas and theorems in the paper.

We calculate the expectation of the mask with some unknown $w^\star$ vector.
\begin{lemma}\label{lemma:expectation}
For $m$ found with \cref{algo:sketch_mask} and $p^0$. Then for $ w^\star \in \mathbb{R}^d$ it holds,
$$ \mathbb{E}\left[ (X^T (w^\star \odot m))_{i} \right] = (X^T w^\star)_{i} $$.
\end{lemma}

\begin{proof}[\cref{lemma:expectation}]
Fix an index $i$ and denote the random variable $Y_t^{0\star} = \frac{X_{i_t i} w_{i_t}^\star}{sp_{i_t}^0}$ ($i_t$ is random and hence $p^0_{i_t}$ is random), it holds that $\sum_{t=1}^s Y_t^{0\star} = \sum_{t=1}^d X_{i_t i} w_{i_t}^\star \cdot \frac{1}{sp_{i_t}^0} =  (X^T(w^\star\odot m))_i$.
$$ \mathbb{E}_{i_t} \left[ Y_t^{0\star} \right] = \sum_{k=1}^d p_{k}^0 \frac{X_{ki} w_k^\star}{sp_k^0} = \frac{1}{s} (X_{(i)}^T w^\star) = \frac{1}{s} (X^T w^\star)_i .$$
Sum over $s$ random variables we have
$$ \mathbb{E}_{i_t}\left[ (X^T(w^\star \odot m))_i \right] = \sum_{t=1}^s \mathbb{E} \left[Y_t^{0\star}\right] =  (X^T w^\star)_i . $$
\end{proof}

\begin{lemma}\label{lemma:variance} For $m$ found with \cref{algo:sketch_mask} with $p^0$. Then for $ w_\star \in \mathbb{R}^d$ it holds,
\begin{align*}
{\rm Var}\left[ \left( X^T(w^\star \odot m) \right)_i \right] = \frac{1}{s} \sum_{k=1}^d \frac{(X_{ki})^2 (w_k^\star)^2}{p_k^0} - \frac{1}{s} \left(X^T w^\star\right)_i^2
\end{align*}
\end{lemma}

\begin{proof}[\cref{lemma:variance}]
According to the variance of independent variables
\begin{eqnarray*}
{\rm Var}\left[ \left( X^T(w^\star \odot m) \right)_i \right] &=&  \sum_{t=1}^s {\rm Var}\left[ Y_t^{0\star}\right] \\ &=& \sum_{t=1}^s \mathbb{E}[(Y_t^{0\star})^2] - \mathbb{E}[Y_t^{0\star}]^2.
\end{eqnarray*}
Then we focus on $\mathbb{E}[(Y_t^{0\star})^2]$ calculation,
\begin{align*}
    \mathbb{E}[(Y_t^{0\star})^2] = \sum_{k=1}^d \frac{(X_{ki})^2 (w_k^\star)^2}{s^2 p_k^0}.
\end{align*}
Then the use of \cref{lemma:variance} concludes the proof.
\end{proof}

\subsection{Proof of \cref{lemma:error_x}}
The following lemma supports \cref{thm:2weight_rand_data} and \cref{thm:NTK}. This lemma establish the variance of each masked entry when the mask is found with $w^0$ and $\Tilde{X}$ but applied on other input data and weight vector.
\begin{replemma}{lemma:error_x}
    Suppose $X,\Tilde{X}\in\mathbb{R}^{d\times n}$, $w^0,w^\star\in\mathbb{R}^d$ and  $s\in \mathbb{Z}^+$ then when using \cref{algo:sketch_mask} with $p^0$ and $\Tilde{X}$ for $m$  the error can be bounded
\begin{eqnarray*}
    && \hspace{-0.3in}\mathbb{E}_{m}\left[\norm{X^T w^\star - X^T (w^\star \odot m)}^2\right] \\ &&= \frac{1}{s} \sum_{k=1}^d \frac{1}{p_k^0}\norm {X_{(k)}}^2 {w_k^\star}^2 - \frac{1}{s} \norm{X^T w^\star}^2 \\
    &&\leq \frac{1}{s} \sum_{k=1}^d \frac{\sum_{j=1}^d \norm{\Tilde{X}_{(j)}}|w^0_j|}{\norm{\Tilde{X}_{(k)}}|w^0_k|} \norm{X_{(k)}}^2 {w^\star_k}^2.
\end{eqnarray*}
\end{replemma}

\begin{proof}
$m$ is drawn i.i.d.
    \begin{eqnarray*}
&& \hspace{-0.3in} \mathbb{E}_{m}\left[ \norm{X^T w^\star - X^T (w^\star \odot m)}^2 \right] \\ && = \sum_{i=1}^n \mathbb{E}_m\left[ (X^T w^\star - X^T (w^\star \odot m))_i^2 \right] \\
&&= \sum_{i=1}^n (X^Tw^\star)_i^2 - 2(X^Tw^\star)_i \mathbb{E}_m \left[ (X^T (w^\star \odot m))_i\right] \\ && + \mathbb{E}_m \left[ (X^T(w^\star \odot m))_i^2  \right]\\
&&\stackrel{{\cref{lemma:expectation}}}{=} \sum_{i=1}^n {\rm Var}_m\left[ (X^T (w^\star \odot m))_i \right]\\
&&\stackrel{{\cref{lemma:variance}}}{=} \frac{1}{s} \sum_{k=1}^d \frac{1}{p_k^0}\left(\sum_{i=1}^n {X_{ki}}^2\right) (w_k^\star)^2 - \frac{1}{s} \norm{X^T w^\star}^2 \\
&&\leq \frac{1}{s} \sum_{k=1}^d \frac{1}{p_k^0}\left(\sum_{i=1}^n {X_{ki}}^2\right) (w_k^\star)^2 \\
&&\stackrel{{p^0\text{ definition}}}{=} \frac{1}{s} \sum_{k=1}^d \frac{\sum_{j=1}^d \norm{\Tilde{X}_{(j)}}|w^0_j|}{\norm{\Tilde{X}_{(k)}}|w^0_k|} \norm{X_{(k)}}^2 {w^\star_k}^2.
\end{eqnarray*}  
\end{proof}

Now, we can repeat the lemmas and theorem from the main paper and present their proofs.

\subsection{Proof of \cref{lemma:sketching_rand_data}}

\begin{replemma}{lemma:sketching_rand_data}
Suppose $X\in\mathbb{R}^{d\times n} \sim \frac{1}{\sqrt{n}} \mathcal{N}(0, I)$, $w^0\in\mathbb{R}^d$ and  $s\in \mathbb{Z}^+$ then when using \cref{algo:sketch_mask} with $p^0$ for $m$  the error can be bounded
\begin{align*}
    \mathbb{E}_{X}\left[\norm{X^T w^0 - X^T (w^0 \odot m)}^2\right] \leq \frac{1}{s} \norm{w^0}^2
\end{align*}
\end{replemma}

\begin{proof}
  \begin{align*}
    & \mathbb{E}_{X}\left[\norm{X^T w^0 - X^T (w^0 \odot m)}^2\right] \\ &= \mathbb{E}_{X}\left[ \mathbb{E}_{m|X}\left[\norm{X^T w^0 - X^T (w^0 \odot m)}^2\right] \right] ]\\
    &\stackrel{\cref{lemma:sketching}}{=} \mathbb{E}_{X}\left[\frac{1}{s} \left(\sum_{k=1}^d \norm{X_{(k)}} \abs{w_k^0}\right)^2 - \frac{1}{s}\norm{X^T w^0} \right] \\
    &\leq \mathbb{E}_{X}\left[\frac{1}{s} \left(\sum_{k=1}^d \norm{X_{(k)}} \abs{w_k^0}\right)^2 \right] \\
    &\leq \frac{1}{s} \mathbb{E}_{X}\left[ \sum_{k=1}^d \norm{X_{(k)}}^2 (w_k^0)^2 \right] \\
    &= \frac{1}{s}  \sum_{k=1}^d \mathbb{E}_{X}\left[\norm{X_{(k)}}^2 \right] (w_k^0)^2 = \frac{n}{sn} \norm{w^0}^2
\end{align*}  
The last inequality is due to the fact that $w$ is deterministic with respect to the expectation and $X_{(k)}$ are independent. 

\end{proof}

\subsection{Proof of \cref{thm:2weight_rand_data}}
\begin{reptheorem}{thm:2weight_rand_data}
Suppose $X\in\mathbb{R}^{d\times n} \sim \frac{1}{\sqrt{n}} \mathcal{N}(0, I)$, $w^0,w^\star\in\mathbb{R}^d$ and  $s\in \mathbb{Z}^+$. Let $c = \max_j \abs{\frac{w^\star_j}{w^0_j}}$. When using \cref{algo:sketch_mask} with $p^0$ for $m$  the error can be bounded as follows
\begin{align*}
    &\mathbb{E}_{X}\left[\norm{X^T w^\star - X^T (w^\star \odot m)}^2\right]  \leq \frac{c^2}{s} \norm{w^0}_1^2 
\end{align*}
\end{reptheorem}

\begin{proof}
We calculate the expectation over $X$. Since $X$ rows are i.i.d. we replace $\mathbb{E}\norm{X_{(k)}}$ in $\mathbb{E}\norm{X_{(.)}}$ since for all $k$ the norm is the same.
\begin{align*}
    &\mathbb{E}_{X}\left[\norm{X^T w^\star - X^T (w^\star \odot m)}^2\right]  
    \\
    & \stackrel{{\cref{lemma:error_x}}}{\le} \mathbb{E}_X \left[\frac{1}{s} \sum_{k=1}^d \frac{\sum_{j=1}^d \norm{X_{(j)}}|w^0_j|}{|w^0_k|} \norm{X_{(k)}} {w^\star_k}^2 \right]\\
    &= \frac{1}{s} \sum_{k=1}^d {w^\star_k}^2 \Big(\mathbb{E}_X\left[\norm{ X_{(k)}}^2\right] \\ & ~~~~~~~~~~~~~+ \frac{1}{\abs{w_k^0}}\sum_{j\ne k}|w^0_j| \mathbb{E}_X\left[\norm{X_{(k)}}
    \right]\mathbb{E}_X\left[\norm{X_{(j)}}\right] \Big) \\
    &= \frac{1}{s}  \sum_{k=1}^d \mathbb{E}_X\left[\norm{ X_{(.)}}^2\right] {w^\star_k}^2 +  \mathbb{E}_X\left[\norm{X_{(.)}}
    \right]^2  \frac{{w^\star_k}^2}{\abs{w_k^0}}  \sum_{j\ne k} |w^0_j|.
\end{align*}
Note that according to the distribution of $X$,
$$ \mathbb{E}_X\left[\norm{ X_{(.)}}^2\right] = \mathbb{E}\left[ \sum_{i=1}^n \frac{1}{n} x_i^2\right] = 1, $$
$$ \mathbb{E}_X\left[\norm{ X_{(.)}}\right]^2 = \frac{1}{n} \mathbb{E} \left[ \sqrt{\sum_{i=1}^n x_i^2} \right]^2 = \frac{2}{n} \frac{\Gamma((n+1)/2)^2}{\Gamma(n/2)^2} \leq 1,$$
where $\Gamma(\cdot)$ is the gamma function.

Calculate the result as a function of $w^0$ and the ratio, $c$:
\begin{align*}
        & \frac{1}{s}  \sum_{k=1}^d  {w^\star_k}^2 +    \frac{{w^\star_k}^2}{\abs{w_k^0}}  \sum_{j\ne k} |w^0_j| = \frac{1}{s}  \sum_{k=1}^d  \frac{{w^\star_k}^2}{\abs{w_k^0}}  \sum_{j =1}^d |w^0_j|\\
        &= \frac{1}{s} \norm{w^0}_1 \sum_{k=1}^d \frac{{w_k^\star}^2}{\abs{w_k^0}} \\
        &\leq \frac{1}{s} \norm{w^0}_1 \sum_{k=1}^d \frac{{w_k^0}^2c^2}{\abs{w_k^0}} \leq \frac{c^2}{s} \norm{w^0}_1^2.
\end{align*}
Note that in the proof we assumed that the weights at initialization were non-zero, which is a very common practice in neural networks weight initialization. 

\end{proof}

\subsection{Proof of \cref{lemma:unimask_rand_data}}
\begin{replemma}{lemma:unimask_rand_data}
Let $X\in\mathbb{R}^{d\times n} \sim \frac{1}{\sqrt{n}} \mathcal{N}(0, I)$, $w^0\in\mathbb{R}^d$, $s\in \mathbb{Z}^+$ and $c = \max_j \abs{\frac{w^\star_j}{w^0_j}}$. Then when choosing $m$ uniformly at random (i.e., using \cref{algo:sketch_mask} with a uniform distribution) the error obeys
\begin{align*}
    & \mathbb{E}_{X}\left[\norm{X^T w^\star - X^T (w^\star \odot m)}^2\right] 
    \leq \frac{d}{s} \norm{w^\star}^2 \leq \frac{dc^2}{s} \norm{w^0}^2
\end{align*}

\end{replemma}
\begin{proof}
    This proof is in the same form as \cref{thm:2weight_rand_data} and we start by establishing the equivalence to \cref{lemma:expectation,lemma:variance}.
    In order to maintain the estimation of the features, when $m_i \ne 0$ it holds that $m_i = \frac{d}{s}$ (as stated in \cref{algo:sketch_mask}, $m_i = \frac{1}{sp_i}$).
    $$\mathbb{E}_m\left[ (X^T (w^\star\odot m)_i \right] =  \sum_{t=1}^s \frac{1}{d} (X^T w^\star)_i \frac{d}{s} = (X^T w^\star)_i$$

    Next we analyse the variance ${\rm Var}_m\left[ (X^T(w^\star \odot m))_i\right]$. Fix an index $i$ and let $Y_t^{Uni\star} = \frac{X_{i_t i} w_{i_t}^\star}{sp_{i_t}^{Uni}} = \frac{dX_{i_t i} w_{i_t}^\star}{s}$.
    We look at
    \begin{eqnarray*}
    {\rm Var}\left[ \left( X^T(w^\star \odot m) \right)_i \right] &=&  \sum_{t=1}^s {\rm Var}\left[ Y_t^{Uni\star}\right] \\ &=& \sum_{t=1}^s \mathbb{E}[(Y_t^{Uni\star})^2] - \mathbb{E}[Y_t^{Uni\star}]^2 .
    \end{eqnarray*}
    
    $$ \mathbb{E}_m \left[ {Y_t^{Uni \star}}^2\right] = \sum_{k=1}^d \frac{X_{ki}^2 (w_k^\star)^2}{s} \frac{d}{s}$$
    Overall we can conclude that 
    $${\rm Var}_m\left[ (X^T(w^\star \odot m))_i\right] = \frac{d}{s} \sum_{k=1}^d X_{ki}^2 (w_k^\star)^2 -\frac{1}{s} (X^T w^\star)_i^2 .$$

    Finally we calculate the expectation over $X$ and we have
     $$\mathbb{E}_{X,m}\left[ \norm{X^T w^\star - X^T(w^\star \odot m)} \right] \leq \frac{d}{s} \norm{w^\star}^2 \leq \frac{dc^2}{s} \norm{w^0}^2. $$
\end{proof}

\subsection{Proof of \cref{thm:NTK}}
\begin{reptheorem}{thm:NTK}
    Under [1-4] assumptions, $\delta_0 > 0$ and $\eta_0\leq \eta_{critical}$.
    Let $m\in\mathbb{R}^d$ be a $s$-dense mask found with \cref{algo:sketch_mask} with $p$ according to $f_0^{lin}(\mathcal{X})$ and $\theta_0$ (\cref{eq:prob}).
    Then there exist $R_0 >0$, $K>1$ such that for every $n\geq N$ the following holds with probability $>1-\delta_0$ over random initialization when applying GD with $\eta_0$
    \begin{align*}
        \mathbb{E}&_m\left[\norm{f_t^{lin}(\mathcal{X}) - \nabla_{\theta_t} f_t(\mathcal{X}) (\theta_t \odot m)}^2\right] \\
        &\leq \frac{1}{s} K^3 \norm{\theta_0}_1 F(J(\theta_0)) \cdot  \\ & \quad \left(\norm{\theta_0}_1 + F(\theta_0) \frac{9K^4R_0^2}{\lambda_{\min}^2} + 6\sqrt{d} \frac{K^3R_0}{\lambda_{\min}}\right),
    \end{align*}
    where $F(A) = \sum_{i=1}^d \frac{1}{\norm{A^{(i)}}}$.
\end{reptheorem}

\begin{proof}
    We start with computing the expected error of approximation of the features at time $t$ using \cref{lemma:error_x}. Note that we use the jacobian at initialization $J(\theta_0)$ and the initial parameters $\theta_0$ to find the mask $m$.
    \begin{align*}
    &\mathbb{E}_m\left[\norm{f_t^{lin}(\mathcal{X}) - \nabla_{\theta_t} f_t(\mathcal{X}) (\theta_t \odot m)}\right] \\ &=
        \mathbb{E}_m \left[ \norm{J(\theta_t)\theta_t - J(\theta_t)(\theta_t \odot m)}^2 \right] \\
        &\leq \frac{1}{s} \sum_{k=1}^d \frac{\norm{J(\theta_t)^{(k)}}^2}{\norm{J(\theta_0)^{(k)}}} \frac{\theta_{tk}^2}{\abs{\theta_{0k}}} \sum_{j=1}^d \norm{J(\theta_0)^{(j)}}\abs{\theta_{0j}} \\
        &\leq \frac{1}{s}\left(\sum_{j=1}^d \norm{J(\theta_0)^{(j)}}\abs{\theta_{0j}}\right) \left( \sum_{k=1}^d \frac{\norm{J(\theta_t)^{(k)}}^2}{\norm{J(\theta_0)^{(k)}}}\right)
        \left( \sum_{k=1}^d \frac{\theta_{tk}^2}{\abs{\theta_{0k}}} \right).
    \end{align*}

    Next we bound each term.
    \begin{align*}
        \sum_{k=1}^d \frac{\theta_{tk}^2}{\abs{\theta_{0k}}} &= \sum_{k=1}^d \frac{(\theta_{tk} - \theta_{0k})^2}{\abs{\theta_{0k}}} +2\theta_{tk}sign(\theta_{0k}) - \abs{\theta_{0k}} \\
        &\leq \left( \sum_{k=1}^d \frac{1}{\abs{\theta_{0k}}} \right) \norm{\theta_t - \theta_0}^2 + 2\norm{\theta_t}_1 - \norm{\theta_0}_1 \\
        &\leq \left( \sum_{k=1}^d \frac{1}{\abs{\theta_{0k}}} \right) \norm{\theta_t - \theta_0}^2 + 2\norm{\theta_t - \theta_0}_1 + \norm{\theta_0}_1,
    \end{align*}
    where the last transition is done we the reverse triangle inequality and subtraction and addition of $\norm{\theta_0}_1$.
    Hence when using the bound in Theorem G.4 and $\norm{a}_1 \leq \sqrt{d}\norm{a}_2$ we have,
    \begin{align*}
        \sum_{k=1}^d \frac{\theta_{tk}^2}{\abs{\theta_{0k}}} &\leq \norm{\theta_0}_1 + \left( \sum_{k=1}^d \frac{1}{\abs{\theta_{0k}}} \right) \left(\frac{3KR_0}{\lambda_{\min}}\right)^2 + 6\sqrt{d} \frac{KR_0}{\lambda_{\min}}.
    \end{align*}
We bound the next term,
\begin{align*}
    &\left(\sum_{j=1}^d \norm{J(\theta_0)^{(j)}}\abs{\theta_{0j}}\right) \left( \sum_{k=1}^d \frac{\norm{J(\theta_t)^{(k)}}^2}{\norm{J(\theta_0)^{(k)}}}\right) \\ &\leq \left(\sum_{j=1}^d \norm{J(\theta_0)^{(j)}}\abs{\theta_{0j}}\right) \left( \sum_{k=1}^d \frac{1}{\norm{J(\theta_0)^{(k)}}}\right) \cdot \\ & ~~~~~~~~~~~~~~~~~~~~~~~~~~~~~~~~~~~~~~~~~~~~~~~~~~~~~~~~~~~~\left( \sum_{k=1}^d \norm{J(\theta_t)^{(k)}}^2\right) \\
    &= \left(\sum_{j=1}^d \norm{J(\theta_0)^{(j)}}\abs{\theta_{0j}}\right) \left( \sum_{k=1}^d \frac{1}{\norm{J(\theta_0)^{(k)}}}\right) \norm{J(\theta_t)}_F^2 \\
    &\leq \left(\sum_{j=1}^d \norm{J(\theta_0)^{(j)}}\abs{\theta_{0j}}\right) \left( \sum_{k=1}^d \frac{1}{\norm{J(\theta_0)^{(k)}}}\right) K^2
\end{align*}
We can also bound,
\begin{align*}
    \left(\sum_{j=1}^d \norm{J(\theta_0)^{(j)}}\abs{\theta_{0j}}\right) &\leq \sum_{j=1}^d \norm{J(\theta_0)}_F\abs{\theta_{0j}} \leq K \norm{\theta_0}_1.
\end{align*}
    Overall we can bound the error with the state of the model at initialization
    \begin{align*}
        &\mathbb{E}_m \left[ \norm{J(\theta_t)\theta_t - J(\theta_t)(\theta_t \odot m)}^2 \right] \leq 
        \\
        &\leq \frac{1}{s} K^3 \norm{\theta_0}_1 \left( \sum_{k=1}^d \frac{1}{\norm{J(\theta_0)^{(k)}}}\right)  \cdot \\ & ~~~~~~~~~~~~~~~~~~ \left(\norm{\theta_0}_1 + \left( \sum_{k=1}^d \frac{1}{\abs{\theta_{0k}}} \right)\left(\frac{3K^2R_0}{\lambda_{\min}}\right)^2 + 6\sqrt{d} \frac{K^3R_0}{\lambda_{\min}}\right)
        \\
        &\leq \frac{1}{s} K^3 \norm{\theta_0}_1 F(J(\theta_0))   \left(\norm{\theta_0}_1 + F(\theta_0) \frac{9K^4R_0^2}{\lambda_{\min}^2} + 6\sqrt{d} \frac{K^3R_0}{\lambda_{\min}}\right),
    \end{align*}
    where $F(A) = \sum_{i=1}^d \frac{1}{\norm{A^{(i)}}}$.
    
\end{proof}

\end{document}